\documentclass[conference]{IEEEtran}
\usepackage{url}  
\usepackage{graphicx}  
\usepackage{subfig}
\usepackage{bm}
\usepackage{dsfont}
\usepackage{algorithm,algpseudocode}
\usepackage{amsmath,amssymb}
\usepackage{amsthm}
\newtheorem{theorem}{Theorem}
 
\newtheorem{lemma}{Lemma}
\newtheorem{assumption}{Assumption}
\newtheorem{remark}{Remark}
\newtheorem{corollary}{Corollay}

\title{\Large  \textbf{A Tractable Algorithm For Finite-Horizon Continuous Reinforcement Learning}}

\begin{document}
\author{\IEEEauthorblockN{Phanideep Gampa}
\IEEEauthorblockA{\textit{Mathematical Science} \\
\textit{Indian Institute of Technology (BHU),}\\
Varanasi, India 221005 \\
gampa.phanideep.mat15@iitbhu.ac.in}
\and
\IEEEauthorblockN{Sairam Satwik Kondamudi}
\IEEEauthorblockA{\textit{Computer Science and Engineering} \\
\textit{Indian Institute of Technology,}\\
Hyderabad, India 502285 \\
satwik7545@gmail.com}
\and
\and
\and
\IEEEauthorblockN{K. Lakshmanan}
\IEEEauthorblockA{\textit{Computer Science and Engineering} \\
\textit{Indian Institute of Technology (BHU),}\\
Varanasi, India 221005 \\
lakshmanank.cse@iitbhu.ac.in}
}

\maketitle

\begin{abstract}
We consider the finite horizon continuous reinforcement learning problem. Our contribution is three-fold.
First,we give a tractable algorithm based on optimistic value iteration for the problem.
Next,we give a lower bound on regret of order $\Omega(T^{2/3})$ for any algorithm discretizes the state space,
improving the previous regret bound of $\Omega(T^{1/2})$ of Ortner and Ryabko \cite{contrl} for the same problem. Next,under the assumption that the rewards and transitions are H\"{o}lder Continuous 
we show that the upper bound on the discretization error is $const.Ln^{-\alpha}T$. 
Finally,we give some simple experiments to validate our propositions.


\end{abstract}

%
%

%
%

\section{Introduction}
In most reinforcement learning algorithms the size of the state and action spaces are assumed to be finite. But many real-world applications have
continuous state or action spaces. Though the case of continuous state space has been considered before like in 
Ortner and Ryabko \cite{contrl} and Lakshmanan et al. \cite{lakshmanan} there is no tractable algorithm whose regret has been analyzed.In this work we consider a simpler finite-horizon setting and give a tractable algorithm with near-optimal regret bound.

There have been many assumptions on the rewards and transition function considered before like deterministic transitions
in \cite{bernstein} or transition functions that are linear in state and action \cite{strehl}, \cite{brunskill}, \cite{abbasi} and \cite{ibrahmi}.
Kakade et al. \cite{kakade} have considered more general setting of PAC learning in
RL with metric state spaces. Another result is by Osband and Roy \cite{osband} where bounds on expected regret was derived when reward and transition functions
belonged to a class of functions characterized by parameters like eluder dimension and Kolmogorov dimension.
In this paper we consider the most general assumption that the reward function and the transition functions are H\"{o}lder continuous like the papers \cite{contrl}
and \cite{lakshmanan}. We derive our bound on the discretization error without any further assumptions. By assuming unbiased sampling, we also derive an upper bound on the regret which is $O(T^{2/3})$  when both transition probabilities and reward function are Lipschitz continuous. 

In recent work, Azar et al. \cite{azar2017minimax} have proposed a new algorithm Optimistic Value Iteration (UCBVI).
This is different from the UCRL algorithm
in Jacksh et al \cite{ucrl} upon which the algorithms of Ortner and Ryabko \cite{contrl} was based.
For episodic problems, it is known that UCBVI algorithm has better regret bounds. We extend this UCBVI algorithm for continuous state space
problems using the state space aggregation used in \cite{contrl}.
It should also be noted that both the algorithms in \cite{contrl} and \cite{lakshmanan} are not tractable. Though we have considered an easier
episodic setting, our algorithm is tractable. And we show that the discretization error(discretizing the state space to intervals) is bounded above by $const.Ln^{-\alpha}T$ under the assumptions that rewards and transitions are H\"{o}lder Continuous.

For infinite horizon setting, in the one-dimensional case with Lipschitz rewards and transition functions
the bound by Ortner and Ryabko is of order $\tilde{O}(T^{3/4})$ in
\cite{contrl}. If the transition functions are in addition smooth, a regret bound of $O(T^{2/3})$ is shown in \cite{lakshmanan}.

The best known lower bound for Lipschitz reward and general transition function is by Ortner and Ryabko \cite{contrl} which is order $\Omega(T^{1/2})$.
We improve upon this by giving a lower bound of $\Omega(T^{2/3})$, matching the upper bound for algorithms that discretizes the state-space.
We have also implemented the algorithm and give experimental results on one and two dimensional setting. The empirical performance is seen to match with
the theoretical results.

\section{Terminology and Problem Definition}
We define Markov Decision Process (MDP) by  state space $S$, action space $A$, reward function $r:S\times A \rightarrow [0,1]$ and
transition probability $P:S\times A \times S \rightarrow [0,1]$.
We consider continuous state space $[0,1]^d$ and finite action MDP. For simplicity we derive results for the one dimensional state space, they can
be easily generalized to higher dimensional state spaces.The random rewards given  state $s$ and action $a$ are bounded in $[0,1]$ with mean $r(s,a)$.
The probability of going to state $s'$ given state $s$ and action $a$ is given by the transition probability $p(s'|s,a)$.
We make the following assumptions like in \cite{contrl}, \cite{lakshmanan}.
This guarantees that rewards and transitions are close in close states.

\begin{assumption} \label{ass:1}
There are $L,\, \alpha > 0$ such that for any two states $s, \, s'$ and all actions a,
\begin{equation}
\big|r(s,a) - r(s',a)\big| \leq L|s-s'|^\alpha.
\end{equation}
\end{assumption}

\begin{assumption}\label{ass:2}
There are $L,\, \alpha > 0$ such that for any two states $s, \, s'$ and all actions a,
\begin{equation}
 \big\| p(\cdot | s,a) - p(\cdot | s',a) \big\|_1 \leq L|s-s'|^\alpha.
\end{equation}
\end{assumption}

For simplicity we assume $L, \, \alpha$ are same for both assumptions.

We consider the finite horizon setting where the agent interacts with the environment in $H$ number of steps per episode.
We denote by $[n]$ the set $\{i\in \mathbb{N}, 1\leq i \leq n\}$).
The policy during an episode is expressed as a mapping $\pi:S\times[H]\to A $.
Let $x_{k,h}$ denote the state in the step $h$ of the episode $k$ and $\pi_k$ denote the policy for the episode $k$.
The value function of each state in $k^{th}$ episode from step $h$ by following a policy $\pi_k$ is defined by
\begin{equation}\label{eq:value}
V^{\pi_{k}}_{h}(s)=\mathds{E}_{\pi_k}\bigg(\sum_{i=h}^{H}r(x_{k,i},\pi_k(x_{k,i},i))|x_{k,h}=s\bigg).
\end{equation}
The \textit{optimal} value function is defined as  $V_{h}^{*}(x)= \sup\limits_{\pi} V_{h}^{\pi}(x)  $ for all $x \in S$ and $h \geq 1$.
The performance of the algorithm is measured according to regret incurred in  all the episodes as given by
\begin{equation}
\label{eq:regret}
\mbox{Regret}(K)=\sum_{k=1}^{K} \big(V_{1}^{*}(x_{k,1})-V_{1}^{\pi_{k}}(x_{k,1})\big).
\end{equation}

We discretize the continuous state space into intervals of length $1/n$ like in \cite{contrl}.
Let $I(s)$ denote interval of the state $s\in S$.
We define \textit{aggregate rewards} and \textit{aggregate transition probabilities} with respect to $\{I_1,I_2,\ldots,I_n\}$ as
\begin{equation}
\label{eq:agg prob}
p^{agg}(I_j|s,a)=\int_{I_j}p(ds'|s,a)\\
\end{equation}
\begin{equation}
\label{eq:agg rew}
r^{agg}(I_j,a)	=  n.\int_{I_j}r(s,a)ds.
\end{equation}

Here $r^{agg}(I_j,a)$ can be interpreted as the mean reward in the interval $I_j$. The algorithm treats each interval as single state and the
aggregate policy is $\pi_{k}^{agg}(I(s),i)$.
The aggregated value function following this policy is $V^{agg}_{\pi_{k},h}(I_{j})$.At any state $s \in S$ and and step $i\in [H]$,the mapping between the policies are given by
\begin{equation}
\label{eq:policy}
\pi_{k}(s,i)=\pi_{k}^{agg}(I(s),i)
\end{equation}
\begin{equation}\label{eq:agg value}
V^{agg}_{\pi_{k,h}}(I)=\mathds{E}_{\pi^{agg}_k}\bigg(\sum_{i=h}^{H}r^{agg}(I_{k,i},\pi^{agg}_k(I_{k,i},i))|I_{k,h}=I\bigg).
\end{equation}
For the discretised MDP, the transitions are between intervals,so we define $P^{agg}(I|I',a)$ as the average of $p^{agg}(I|t,a)$ in $I'$
\begin{equation}
\label{eq:int prob}
P^{agg}(I|I',a)= n.\int_{I'}p^{agg}(I|t,a)dt
\end{equation}

\section{Algorithm}
In UCRL \cite{ucrl} based algorithms confidence sets are built around rewards and transition probabilities.
But here we build the confidence set around optimal value function $V^{*}_{h}$ of the discretized MDP as in Azar et al (\cite{azar2017minimax}).
The algorithm proceeds similar to UCBVI in \cite{azar2017minimax}.
The only difference is that here we use aggregated rewards, transitions probabilities and value functions.
The bonus (Algorithm \ref{algo:bonus}) which is used in calculating the 
Q-Values is built from the empirical variance of the estimated next values. This relies on the Bernstein-Freedman's concentration inequality for
building the concentration sets.

In UCRL based algorithms \cite{ucrl} we need to find the optimistic MDP and optimal policy for that MDP. This step is not tractable when the state
space is continuous as we need knowledge of the bias span $H$ (section 3 of \cite{contrl}). 
Though this is not needed for finite-horizon problems, to the best of our knowledge, UCCRL algorithm is not computationally tractable.
We note that in our algorithm which is based on UCBVI, this step is not needed. Briefly the algorithm consists of the following parts/sections.
\begin{enumerate}
    \item Initialization part which also includes the discretizing the continuous MDP based on the input parameters. This step is executed only once.
    \item The next part consists of an iterative flow of three processes.(The number of iterations is equal to the number of episodes $K$)
    \begin{itemize}
        \item Estimating the transition probabilities based on the history till that iterate. 
        \item Finding the Q values using a modified bellman operator(which includes bonus) according to the current transition probabilities.This part can be considered as a simple Dynamic Programming(DP) problem for finding Q values.
        \item Execute the current(discrete) policy(according to the Q values found) and record the feedback given by the environment.
    \end{itemize}
\end{enumerate}

\subsection{Bonus}
This algorithm follows the heuristic principle known as \textit{optimism in the face of uncertainty}. The algorithms in this paradigm employ optimism to guide exploration. In simple words, the algorithm assigns higher bonus value to that action(given a state/interval) which it is most uncertain about. This encourages the exploration of that action next time when it visits the same state again. Indirectly as the number of times the action chosen increases the algorithm would become more and more certain about it. Given a (state/interval,action) pair, the bonus in the algorithm depends on  the variance of value function of the next possible states. Larger the variance,larger the uncertainty of taking that action. \\
\\
Let us understand this using a simple example. Consider a discretized MDP with 4 possible actions. Given a particular interval $I_k$, let the Q values and the corresponding bonus found by the algorithm be equal to those given in the figure \ref{fig:bonus}. In the case of not using bonus, the best action would be $a1$ because $(I_k,a1)$ has the highest Q-value in $\{(I_k,a_i)| i=1,2,3,4\}$. But when bonus is considered, the action $a4$ is the most likely to be chosen. According to the bonus, the decreasing order of uncertainties of actions are $a4,a2,a1,a3$. The combined effect of Q value and the bonus will lead to the exploration of action $a4$. Due to this exploration of action $a4$, the uncertainty about choosing $a4$ reduces.   

\begin{figure}
    \centering
    \includegraphics[width=3.5in,height=2.3in]{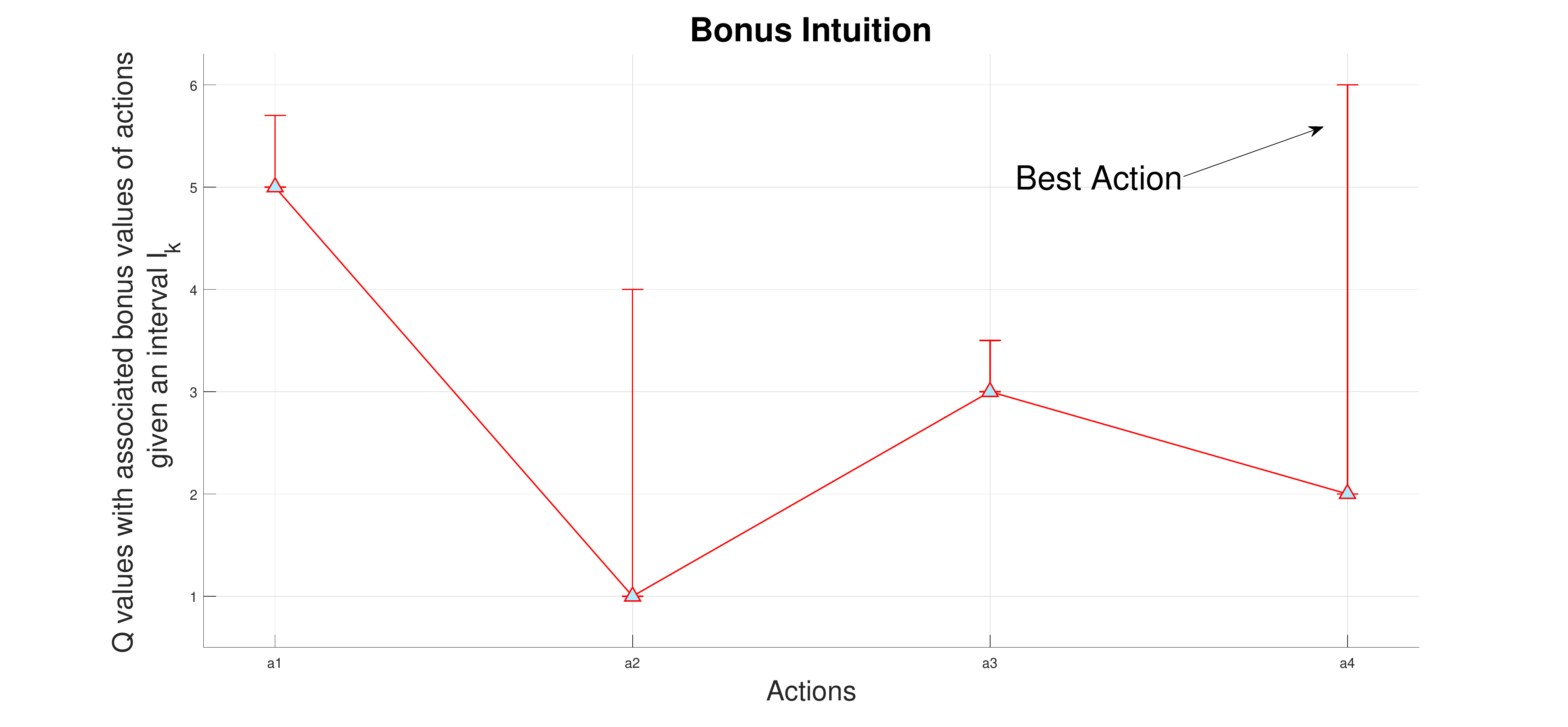}      \caption{Q values and bonus of four actions when an interval is fixed for an example MDP}
    \label{fig:bonus}
\end{figure}
\begin{algorithm}
\caption{UCBVI-CRL}\label{algo:ucbvicrl}
\begin{algorithmic}[1]
\Statex \textbf{Input:} State space $[0,1]$, an action space $A$, discretization parameter $n\in N$,Horizon length $H$, H\"{o}lder parameters $L,\alpha$,
confidence parameter $\delta$, reward function $r$ of continuous MDP.
\Statex \textbf{Initialization:} Let $I_{1} =\left[0,\frac{1}{n}\right] $ and $I_{j} =\left(\frac{j-1}{n},\frac{j}{n}\right]$ for $j=2,3,4 ,\ldots,n$.
Set $t=1$, $\hat{H}=\phi$ and observe the initial state $s_1$ and interval $I(s_1)$. 
Set $S^{*}=\{I_1, I_2,\ldots,I_n\}$. For all $x\in S^*$, $N_k(x,a)$ is number of times action $a$ has been chosen in a state $\in x$.
Similarly with $N_k(x,a,y),N'_{k,h}(x,a)$ where $x,y \in S^*$.
\For {$k=1,2,\ldots$}
\Statex \textbf{Initialize Episode k:}
\Statex  Compute for all $(x,a,y)\in S^{*}\times A\times S^{*}$,
\Statex  $N_k(x,a,y)=\sum_{(x',a',y')\in \hat{H} }\mathbb{I}(x'=x,a'=a,y'=y)$
\Statex  $N_k(x,a)=\sum_{y\in S^{*}}N_k(x,a,y)$
\Statex  $N'_{k,h}(x,a)=\sum_{(x_{i,h},a_{i,h},x_{i,h+1})\in \hat{H}}\mathbb{I}(x_{i,h}=x,a_{i,h}=a)$
\Statex  Let $\kappa=\{(x,a) \in S^{*} \times A, \, N_{k}(x,a)>0\}$
\Statex  Estimate $\hat{P}^{agg}_{k}(y|x,a)=\frac{N_{k}(x,a,y)}{N_{k}(x,a)}, \, \forall \, (x,a) \in \kappa$, using all samples from states interval $y$.
\Statex  Initialize $V^{agg}_{k,H+1}(x)=0$ for all $x\in S^{*}$
\Statex\textbf{Calculating $Q$-values:}
\For{$h=H,H-1,\ldots1$}
\For{$(x,a)\in S^* \times A$}
\If{$(x,a)\in \kappa$}
\Statex \hspace{\algorithmicindent}
$b_{k,h}(x,a)=$bonus$\Big(\hat{P}_k^{agg},V^{agg}_{k,h+1},N_k,N'_{k,h}\Big)$
\Statex \hspace{\algorithmicindent} 
$Q^{agg}_{k,h}(x,a)=\min(Q^{agg}_{k-1,h}(x,a),H,r^{agg}(x,a)+(\hat{P}^{agg}_{k}V^{agg}_{k,h+1})(x,a)+b_{k,h}(x,a))$
\Else\Statex \hspace{\algorithmicindent} \hspace{\algorithmicindent} \hspace{\algorithmicindent} {$Q^{agg}_{k,h}(x,a)=H$}
\EndIf
\Statex \hspace{\algorithmicindent} \hspace{\algorithmicindent} $V_{k,h}^{agg}(x)=\max_{a\in A}\big(Q^{agg}_{k,h}(x,a)\big)$
\EndFor
\EndFor
\Statex \textbf{Executing the policy:}
\For{$h=1,2\ldots H$}
\Statex \hspace{\algorithmicindent} Take action $a_{k,h}=\arg\max_{a}\big(Q^{agg}_{k,h}(x_{k,h},a)\big)$
\Statex \hspace{\algorithmicindent} Update $\hat{H}=\hat{H}\cup(x_{k,h},a_{k,h},x_{k,h+1})$
\EndFor
\EndFor
\end{algorithmic}
\end{algorithm}

\begin{algorithm}
\caption{Bonus}\label{algo:bonus}
\begin{algorithmic}[2]
\Statex \nonumber \textbf{Input:}  $\hat{P}^{agg}_{k}(x,a),V^{agg}_{k,h+1},N_k,N'_{k,h}$, confidence parameter $\delta$.
\Statex \begin{align}
\nonumber \mbox{Let } &b(x,a) = \sqrt{\frac{8LVar_{Y\sim \hat{P}^{agg}_{k}(.|x,a) }(V^{agg}_{k,h+1}(Y))}{N_{k}(x,a)}}\\ &+\frac{14HL}{3N_{k}(x,a)}\\
&+\sqrt{\frac{8\sum_{y}\hat{P}^{agg}_{k}(y|x,a)[\min(\frac{100H^3S^2AL^2}{N'_{k,h+1}(y)},H^2)] }{N_{k}(x,a)}} \label{eq:bonus}
\end{align}\\
where $L=\ln(\frac{5nAT}{\delta})$
\Statex \textbf{return b}
\end{algorithmic}

\end{algorithm}
\section{Results}
The regret analysis of UCBVI-CRL is straight forward.
The regret of the continuous MDP upto K episodes is defined in equation \eqref{eq:regret}.
Now add and subtract terms $V^{agg^*}_{1}(I_{k,1})$, $V^{agg}_{\pi_{k,1}}(I_{j})$ to this to get,
\begin{align*}
\mbox{Regret} (K)=\sum_{k=1}^{K}& \big([V_{1}^{*}(x_{k,1})- V^{agg^*}_{1}(I_{k,1})\\
&+
V^{agg}_{\pi_{k,1}}(I_{k,1})-V_{1}^{\pi_{k}}(x_{k,1})] \\
&+\underbrace{[V^{agg^*}_{1}(I_{k,1})-V^{agg}_{\pi_{k,1}}(I_{k,1})]}_{\Delta_{agg}}\big)\\
\Delta_{error}=\sum_{k=1}^{K}& \big([V_{1}^{*}(x_{k,1})- V^{agg^*}_{1}(I_{k,1})\\
&+
V^{agg}_{\pi_{k,1}}(I_{k,1})-V_{1}^{\pi_{k}}(x_{k,1})]
\end{align*}
\normalsize
Here $\Delta_{agg}$ is the regret of the discretized MDP and 
$\Delta_{error}$ is the error which is a consequence of discretizing the state space.

\subsection{Bounding $\Delta_{error}$}
\begin{lemma}
\label{thm:1}
Using the discretization technique mentioned above, $\Delta_{error}$ is bounded above by $\bm{const.Ln^{-\alpha}T}$.
\end{lemma}

\begin{proof}
Let us split the terms in $\Delta_{error}$ into two parts as
\footnotesize
\begin{align*}
q_{1}^k &= V_{1}^{*}(x_{k,1})-V^{agg^*}_{1}(I_{k,1}) \mbox{ and }\\ q_{2}^k &= V^{agg}_{\pi_{k,1}}(I_{k,1})-V_{1}^{\pi_{k}}(x_{k,1})\\
\Delta_{error}&\leq\sum_{k=1}^{K}\big(|q_1^k|+|q_2^k||\big)\text{\hspace{1cm}(Triangle Inequality)} 
\end{align*}
\normalsize
Consider $|q_{1}^k|$
\footnotesize
\begin{align*}
V_{1}^{*}(x_{k,1})&=\max_{a\in A}\big(r(x_{k,1},a)+\int_{S}p(ds'|x_{k,1},a)V_{2}^{*}(s')\big)\\
&(\mbox{Chapter 4 of \cite{hernandez2012discrete}})
\end{align*}
\begin{align*}       
V^{agg^*}_{1}(I_{k,1}) =\max_{a\in A}&\big(r^{agg}(I_{k,1},a)\\&+\sum_{j=1}^nP^{agg}(I_j|I_{k,1},a)V^{agg^*}_{2}(I_j) \big)\\
&\hspace{1cm}(\mbox{Bellman Equation})
\end{align*}
\begin{align*}
\mbox{Using the property } |\max f&- \max g| \leq \max|f-g|\\ 
|q_{1}^k|\leq\max_{a\in A}\Big|\big[r(x_{k,1},a)&-r^{agg}(I_{k,1},a)\big]\\
+\big[\int_{S}p(ds'|x_{k,1},a)&V_{2}^{*}(s')\\&-\sum_{j=1}^nP^{agg}(I_j|I_{k,1},a)V^{agg^*}_{2}(I_j)\big]\Big|
\end{align*}
\begin{align*}
\mbox{Let }l_1^k&=\max_{a\in A}\Big|r(x_{k,1},a)-r^{agg}(I_{k,1},a)\Big| \mbox{ and }\\
 l_2^k&=\max_{a'\in A}\Big|\int_{S}p(ds'|x_{k,1},a')V_{2}^{*}(s')\\
 &-\sum_{j=1}^nP^{agg}(I_j|I_{k,1},a')V^{agg^*}_{2}(I_j)\Big|
\end{align*}
\begin{align*}
\mbox{Using the property } &\max|f+g|\leq \max |f|+ \max |g| \\ 
|q_{1}^k|&\leq(l_1^k+l_2^k)
\end{align*}
\normalsize
For any action $a\in A$, $l_1^k$ is bounded by
\footnotesize
\begin{align*}
\big|r(x_{k,1},a)&- r^{agg}(I_{k,1},a)\big|=\big|r(x_{k,1},a)-  n.\int_{I_{k,1}}r(s,a)ds\big|\\& \text{ (from \eqref{eq:agg rew})}\\
&\leq\bigg|r(x_{k,1},a)-n.\bigg( \sum_{i=1}^nm_i\Delta x_i\bigg) \bigg|\\ 
&=\bigg|r(x_{k,1},a)-n.\bigg( \sum_{i=1}^nr(s'_{i},a)\frac{1}{n^2} \bigg) \bigg|\\
&\leq\frac{1}{n}.\sum_{i=1}^{n}\big| r(x_{k,1},a)-r(s'_{i},a)\big|
\leq Ln^{-\alpha}.
\end{align*}
\normalsize
The first inequality is obtained by replacing the integral with lower Riemann sum by dividing $I_{k,1}$ into n
intervals each of length $\Delta x_i=\frac{1}{n^2}$ and $m_i$ being the infimum in the sub-interval.
The second inequality follows since function $r$ is continuous w.r.t $s$ and infimum is the minimum. The last inequality follows from
the assumption \eqref{ass:1} and length of $I_j=\frac{1}{n}$.\\

For any action $a\in A$, $l_2^k$ is bounded by
\footnotesize
\begin{align*}
&\leq\big|H\sum_{j=1}^n\int_{I_j}p(ds'|x_{k,1},a)-\sum_{j=1}^nP^{agg}(I_j|I_{k,1},a)V^{agg^*}_{2}(I_j)\big|\\&\text{\hspace{0.5cm}(rewards $\in [0,1]$)} \\
&\leq\sum_{j=1}^n\big|H\int_{I_j}p(ds'|x_{k,1},a)-P^{agg}(I_j|I_{k,1},a)V^{agg^*}_{2}(I_j)\big|\\
&\text{Using }
\big|ab-cd\big|\leq\big|(a+c)(b-d)\big|+\big|bc-ad\big|\\
&a=:H \mbox{ and } b=:\int_{I_j}p(ds'|x_{k,1},a)\\
&c=:V^{agg^*}_{2}(I_j) \mbox{ and } d=:P^{agg}(I_j|I_{k,1},a)
\end{align*}
\begin{align*}
&\leq\sum_{j=1}^n\Big[|(V^{agg^*}_{2}(I_j)+H)\\&\hspace{1cm}(\int_{I_j}p(ds'|x_{k,1},a)-P^{agg}(I_j|I_{k,1},a))|\\&\hspace{1cm}+|V^{agg^*}_{2}(I_j)\int_{I_j}p(ds'|x_{k,1},a)-P^{agg}(I_j|I_{k,1},a)H|\Big]\\
&l_2^k\leq\sum_{j=1}^n\Big[3H\big|\int_{I_j}p(ds'|x_{k,1},a)-P^{agg}(I_j|I_{k,1},a)\big|\Big]\\
&\hspace{1cm}(V^{agg^*}_{2}(I_j)\leq H-1)
\end{align*}
\normalsize
From \eqref{eq:int prob}, replacing the integral with lower Riemann sum with n sub intervals and using the property of continuity that infimum of a function is equal to minimum in that closed interval.\\
\footnotesize
\begin{align*}
l_2^k&\leq\sum_{j=1}^n\Big[3H\big|\int_{I_j}p(ds'|x_{k,1},a)-n.\big(\sum_{i=1}^np^{agg}(I_j|s_i',a)\frac{1}{n^2}\big)\big|\Big]\\
&\leq\sum_{j=1}^n\Big[\frac{3H}{n}\sum_{i=1}^n\big|\int_{I_j}p(ds'|x_{k,1},a)-p^{agg}(I_j|s_i',a)\big|\Big]\\
&=\sum_{j=1}^n\Big[\frac{3H}{n}\sum_{i=1}^n\big|\int_{I_j}p(ds'|x_{k,1},a)-\int_{I_j}p(ds'|s_i',a)\big|\Big]\\
&=\frac{3H}{n}\sum_{i=1}^n\sum_{j=1}^n\big|\int_{I_j}p(ds'|x_{k,1},a)-\int_{I_j}p(ds'|s_i',a)\big|\\
&=\frac{3H}{n}\sum_{i=1}^n\big\| p(\cdot | x_{k,1},a) - p(\cdot | s_i',a) \big\|_1\\
&\leq3HLn^{-\alpha}\hspace{2cm} (x_{k,1},s_i'\in I_{k,1}\forall i \mbox{ and using } \eqref{ass:2})
\end{align*}
\normalsize
Now we have,
\footnotesize
\begin{align*}
 |q_{2}^k| = \big|\mathds{E}_{\pi^k} \Bigg( \sum_{h=1}^{H}(r(x_{k,h},\pi_k)&-r^{agg}(I_{k,h}',\pi_k)) \Bigg)\big|\hspace{1cm}\\&(\mbox{Using }\eqref{eq:policy})
\end{align*}
\normalsize
Hence  $|q_{2}^k|$ can be bound by bounding the difference of the rewards inside the expectation which is similar to bounding $l_1^k$.
For any action $a$ we have,
\footnotesize
\begin{align*}
|q_{2}^k|\leq HLn^{-\alpha}
\end{align*}
\normalsize
Now $\Delta_{error}$ is bounded by
\footnotesize
\begin{align*}
\Delta_{error}=\sum_{k=1}^{K}(|q_{1}^k|+|q_{2}^k|)&\leq \sum_{k=1}^{K} (4H+1)Ln^{-\alpha}\\
 &\leq 5Ln^{-\alpha}T \text{ (as $T\geq(KH)$)}
\end{align*}
\end{proof}
\section{\textbf{Additional Result}: Upper Bound on Regret}
We derived a tighter upper bound on the regret for continuous case(when using the proposed algorithm) when the transition estimates found by the algorithm are unbiased estimates of \eqref{eq:int prob}.The below theorem states that.

\begin{theorem}
Let M be an MDP with continuous state space $[0,1]$, \bm{$A$} actions, known reward function \bm{$r$} and unknown transitions satisfying 
assumptions \ref{ass:1},\ref{ass:2} and the transition estimates are unbiased estimates of \eqref{eq:int prob}. Then with probability $1-\delta$, the regret of UCVBI-CRL after T steps is upper bounded by  
\begin{equation}\tag{6}\label{fullreg}
5Ln^{-\alpha}T+30L'\sqrt{nATH}+2500H^2n^2AL^2+4H\sqrt{TL'}
\end{equation}
where $L' = ln(5HnAT/\delta)$. 
\end{theorem}
\begin{proof}
The regret in the continuous setting can be split into $\Delta_{error}$ and $\Delta_{agg}$.
The bound on $\Delta_{error}$ comes from Lemma \ref{thm:1} and bound on $\Delta_{agg}$ is from 
Theorem 2 of \cite{azar2017minimax}.
\end{proof}

\begin{corollary}
We set $n=T^{x}$. It can be seen that when $\alpha \in [1/k,1]$, setting $x < \frac{k}{2k+1}$ 
first term dominates while the third term $2500H^2n^2AL^2$ dominates when $x >= \frac{k}{2k+1}$.
Thus the optimal regret of $O(L^2H^2AT^{2k/2k+1})$ is obtained by setting $x=\frac{k}{2k+1}$. 
This is better than the regret bound of order $T^{\frac{2k+1}{2k+2}}$ for the same in 
\cite{contrl}. And in the Lipschitz case i.e., when $\alpha =1$ we have regret of order
$O(T^{2/3})$. 
\end{corollary}

\begin{remark}
 The dependence on $H$, $L$ and $A$ is not optimal. If the algorithms of \cite{contrl} and \cite{lakshmanan} are adapted for finite-horizon setting
 We again wish to emphasize the fact that our algorithm is tractable
 unlike the algorithms in the other two papers.
\end{remark}

\begin{remark}
 The regret bound in \cite{lakshmanan} is also of order $\tilde{O}(T^{2/3})$ for the infinite horizon problem in the Lipschitz case.
 But there is an additional parameter  $\beta$ which depends on the smoothness of the transition function. And only in
 the asymptotic case when $\beta \rightarrow \infty$ the regret of $\tilde{O}(T^{2/3})$ is attained.
\end{remark}

\begin{remark}
 For $d$-dimensional state space we have $n^d$ states in the last three terms of equation \eqref{fullreg}. So setting $n=T^{\frac{k}{2kd+1}}$ we get
 a regret of $\tilde{O}(T^{\frac{2kd}{2kd+1}})$ which is better than the bound of $\tilde{O}(T^{\frac{2kd+1}{2kd+2}})$ in \cite{contrl} for the case
 $\alpha \in [1/k,1]$.
\end{remark}

\section{Lower Bounds}

We have the following theorem giving the lower bound.
\begin{theorem}
 For any algorithm $\mathcal{A}$ using state-space discretization, any natural numbers $A$ and $T \geq A$, 
 the expected regret for $\mathcal{A}$ after $T$ timesteps for the MDP $\mathcal{M}$ with state $S=[0,1]$, actions $A$ defined above is
 \begin{equation}
  E(\Delta_T(\mathcal{A})) \geq L. \sqrt{HA} T^{2/3}
 \end{equation}
\end{theorem}
\begin{proof}

We note that any algorithm using state-space discretization also incurs a regret by choosing a wrong action due to
discretization error. Hence we have 
\begin{enumerate}
\item Regret in the discretized MDP
\item Regret due to the discretization error 
\end{enumerate}

The total regret is atleast maximum of these two.
The first term is of order $\sqrt{nHAT}$ (See \cite{ucrl}).
Now consider the MDP shown in figure \ref{fig:lmdp} with two actions in each state. 

\begin{figure}[h]
\centering
\includegraphics[width=2.5in]{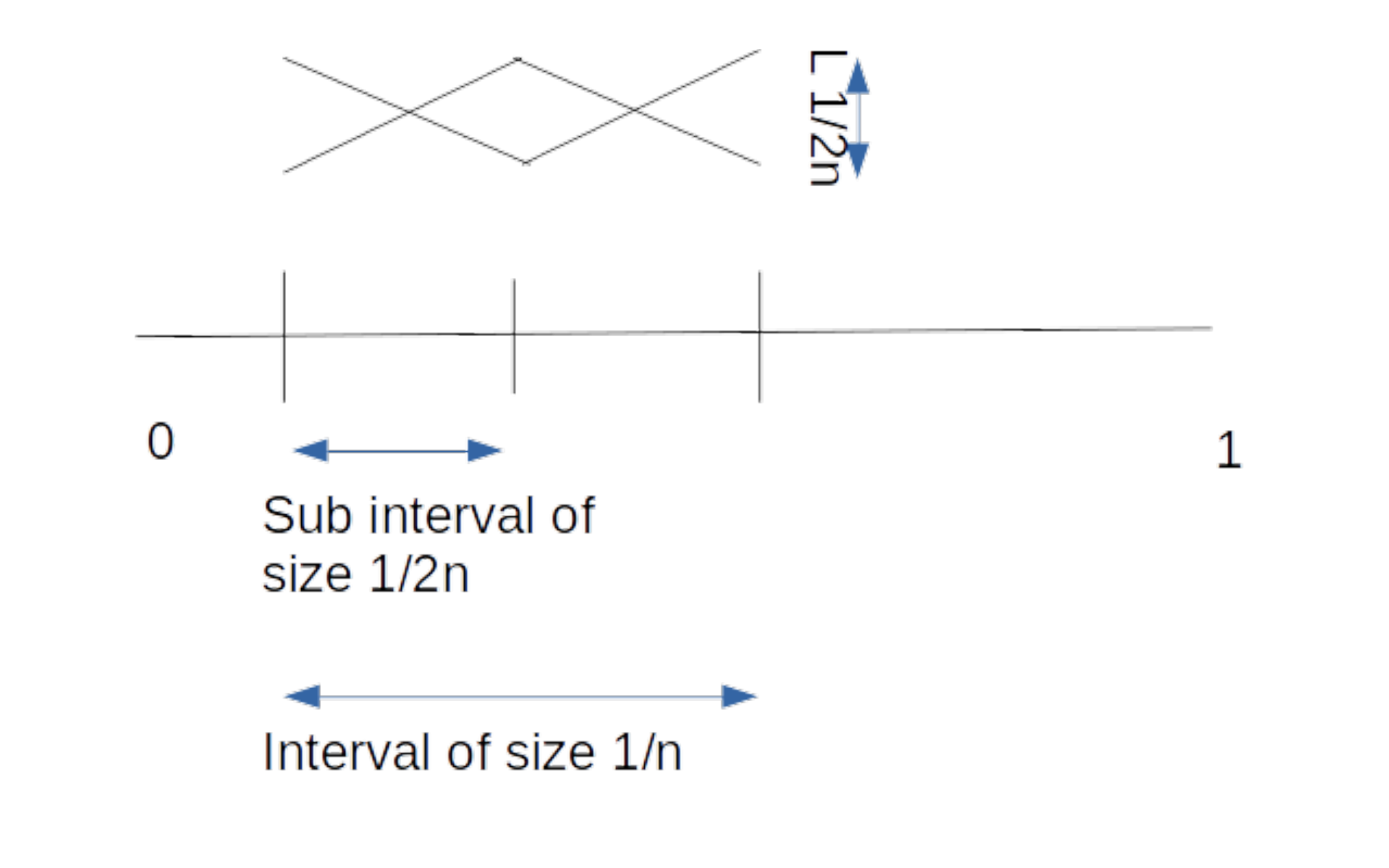}
\caption{State space and reward for the MDP}
\label{fig:lmdp}
\end{figure}
  
Let the algorithm discretize the state into $n$ intervals.
Let us assume that there are two sub-intevals in each interval and the optimal
action in each sub-interval is different. Assume for simplicity that all intervals are of equal size. Otherwise divide each interval into
sub-intervals of size half of the least interval size and assume constant reward for the remaining length of the intervals.

The mean rewards for actions are as shown in the figure \ref{fig:lmdp}, i.e., they are fixed for a state.
It can be seen that it is not possible to determine the optimal action for
both these sub-intervals simultaneously. And the regret obtained due to discretization
is of order $L.\frac{1}{2n}$ as the the rewards are Lipschitz continuous.
Since this regret is incurred on an average for atleast half of the states visited $T/2$. The total regret is of order $L.\frac{T}{4n}$.
Balancing this and the regret in the discretized MDP by taking $n=T^{1/3}$ we get a regret of order $\Omega(L.\sqrt{HA}T^{2/3})$.
\end{proof}

\begin{remark}

We note that algorithms like value iteration and non-stationary value iteration algorithms (Chapter 3 of \cite{mdpbook1})
which do not use state-space discretization cannot be used here as the probability
transition function is not known.
\end{remark}
\begin{remark}
It is straight forward to show a regret bound of $\Omega(L.\sqrt{H_{sp}A}T^{2/3})$ for infinite horizon setting where $H_{sp}$ is the span of the 
optimal bias function.
\end{remark}

\begin{figure}[h]
\centering\includegraphics[scale=0.18]{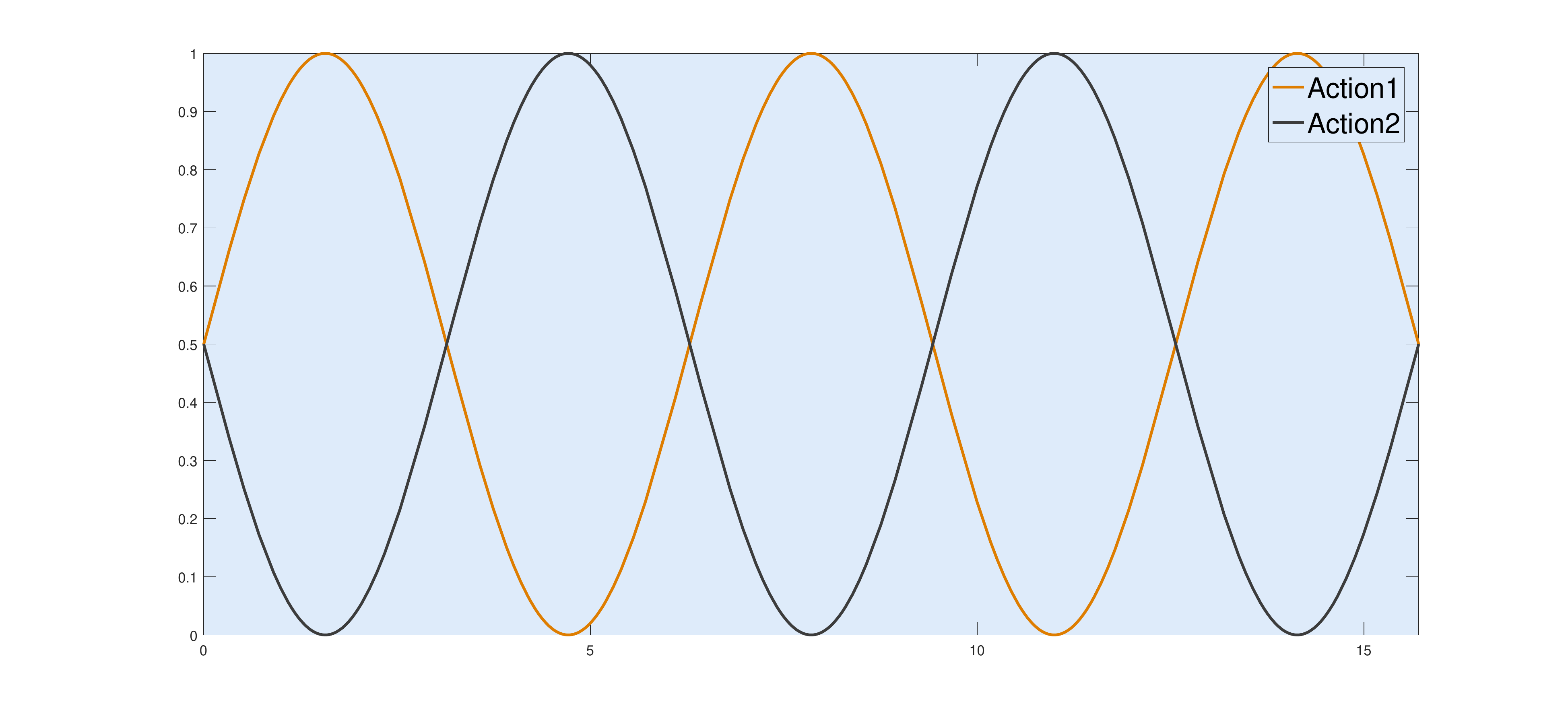}
\caption{1-D Reward Functions}
\label{fig:r-1d}
\centering\includegraphics[scale=0.18]{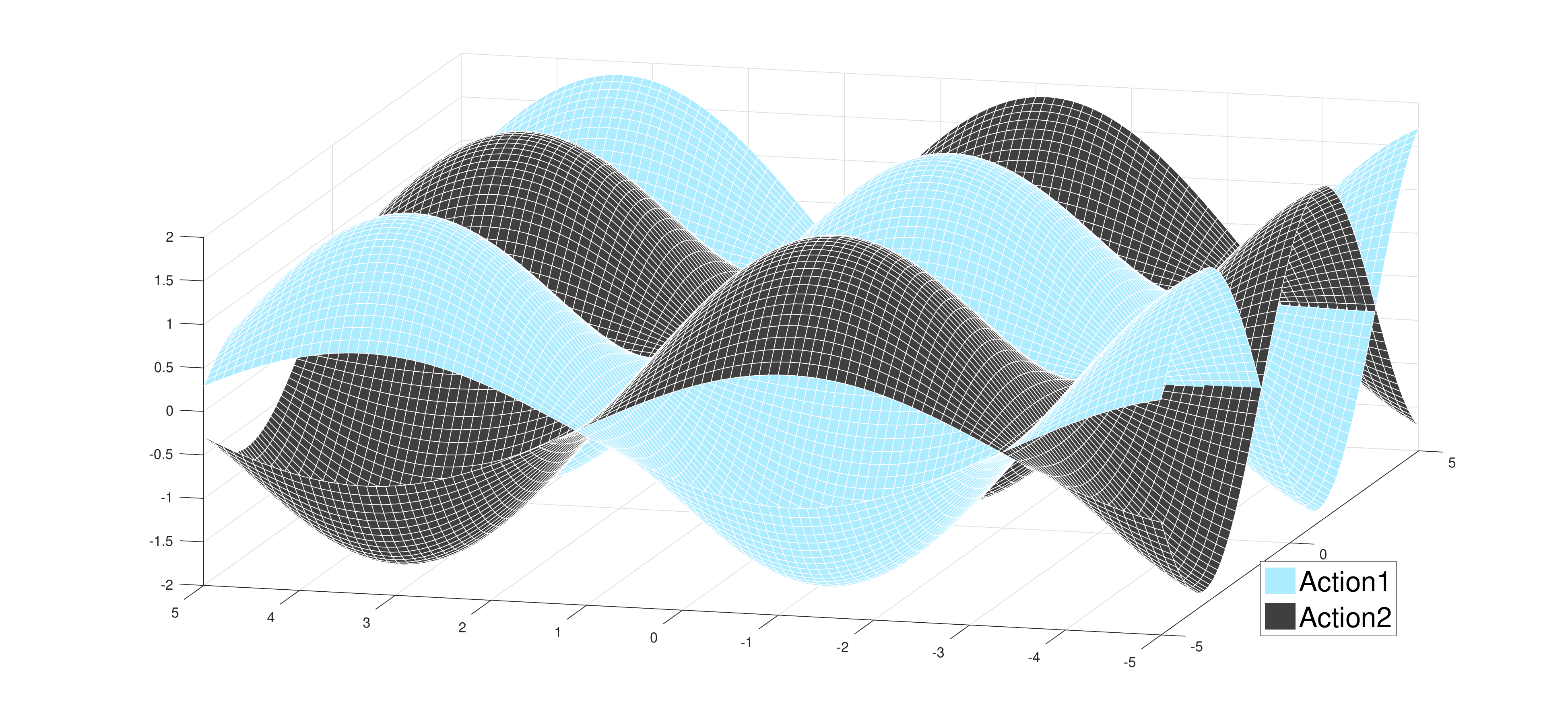}
\caption{2D reward functions plot}\label{fig:r-2d}
\end{figure}

\section{Experiments}
We have performed our experiments on one dimensional and two dimensional state spaces. 
For one-dimensional case we implemented the algorithm on a simple problem with $A=2$ with $[0,5\pi]$ being the state space $S$.
In this problem, the reward functions for the actions $a_1, a_2$ are $\sin(x)$ and 
$-\sin(x)$ respectively. 
So the optimal policy is, for a state $x$ in interval $[n\pi,(n+1)\pi]$, the one which takes action $a_2$ if $n$ is odd and 
\begin{figure}[H]
\begin{tabular}{c}
\subfloat[]{\includegraphics[width = 3in]{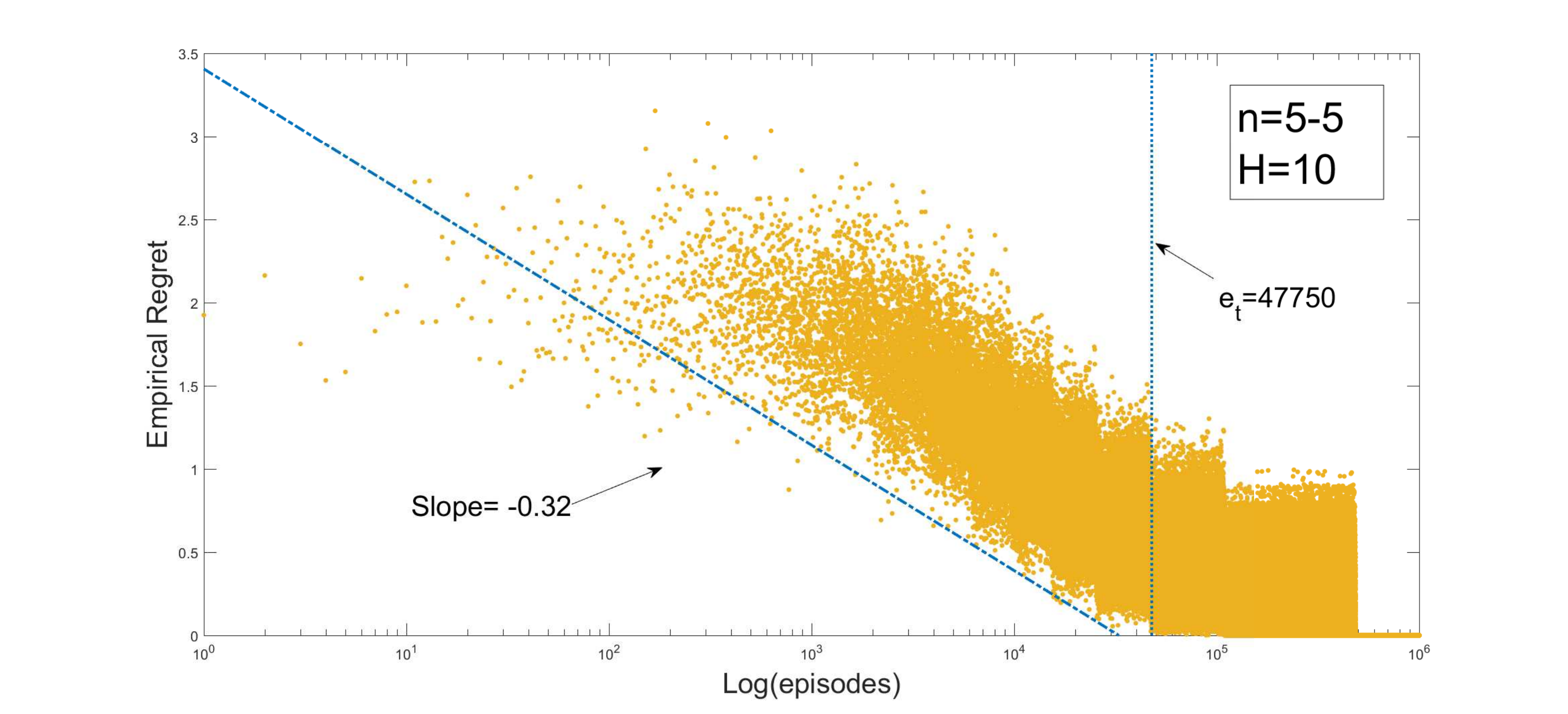}}\\
\subfloat[]{\includegraphics[width = 3in]{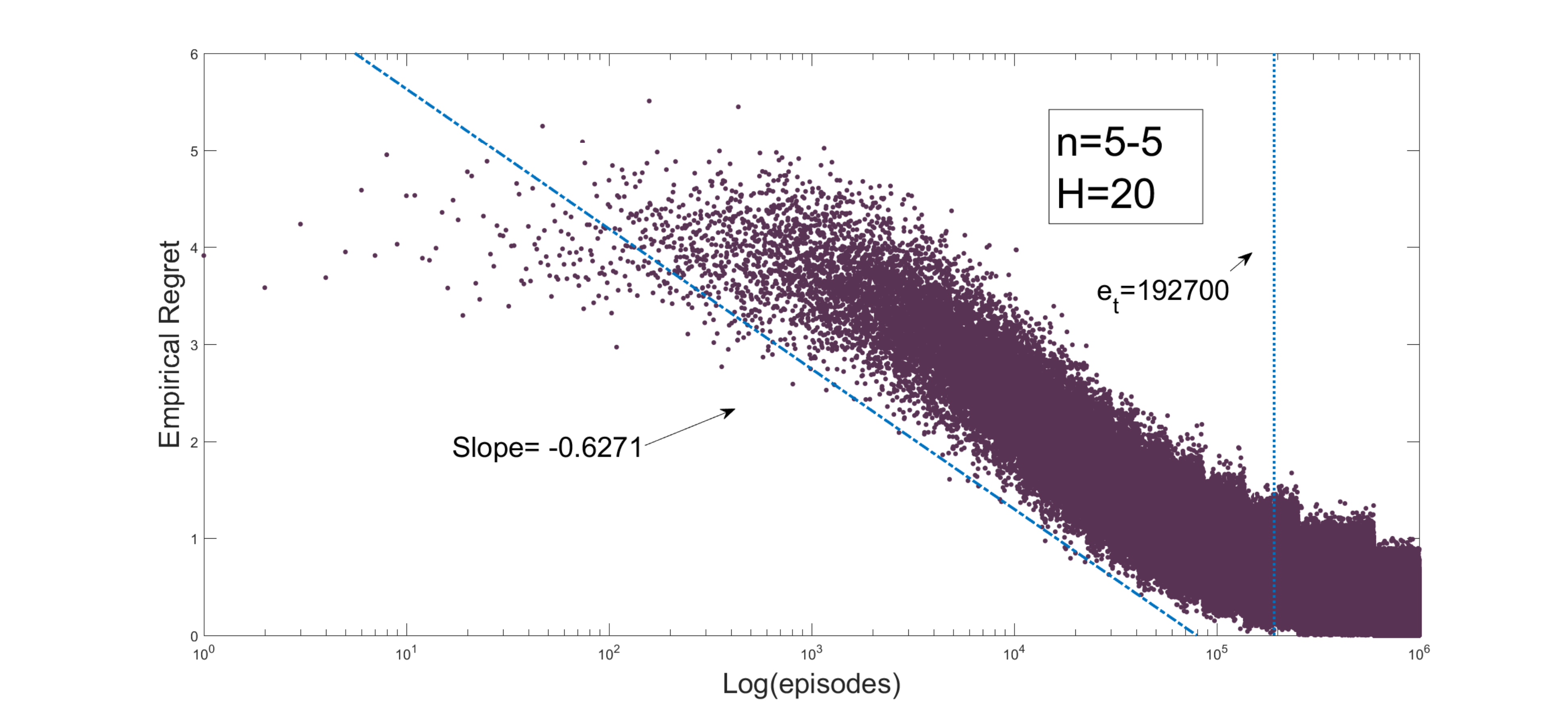}}\\
\subfloat[]{\includegraphics[width = 3in]{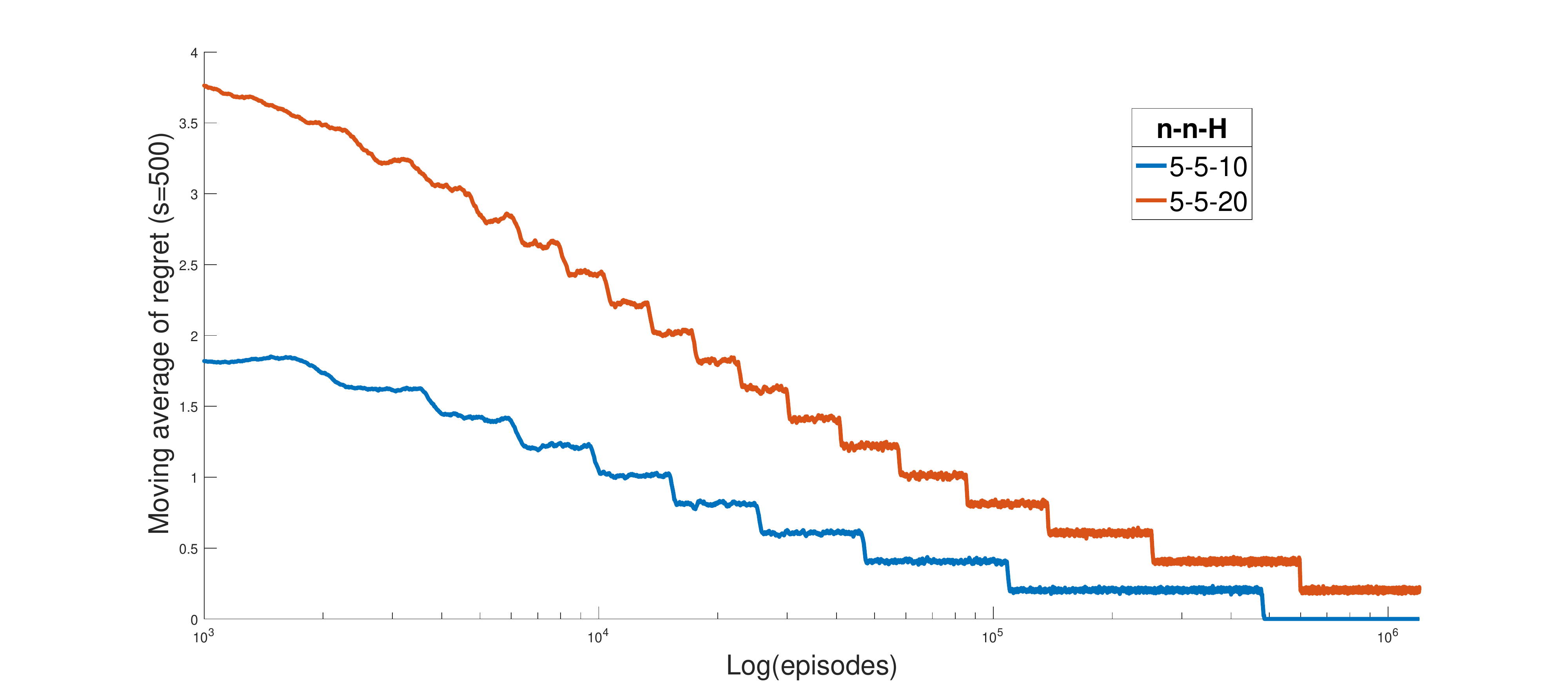}}
\end{tabular}
\caption{Plots (a), (b) show the regret on log scale of the two-dimensional problem for different settings. Plot (c) is obtained by taking 
moving average on the regret.}
\label{fig:regret-2d}
\end{figure}

$a_1$ if $n$ is even.
The reward was scaled to the range $[0,1]$ to meet the requirements of \cite{azar2017minimax}.
From the current state $x$, after taking an action $a_{i}$ such that $i \in{1,2}$, the next state is sampled uniformly from the state space $S$.

We can see from Figure \ref{fig:regret-1d} that the empirical regret is converging faster for larger number of intervals, for the same horizon length
The convergence can be better understood by comparing the the points where the 
empirical regret approaches to $0$. The very first point, in every plot, where the empirical regret is zero is indicated by a dashed vertical
line showing the episode number. 
The results match the intuition that for the same problem having more number of states improves the performance.

By keeping $n$ constant and varying $H$ it can be seen from the plots that regret approaches zero slower as $H$ is increases.
It can also be seen that the regret for larger horizon length at a particular episode was higher than the regret at the same episode for smaller horizon
length. Thus the regret increases with the horizon length as expected.

\begin{figure}[H]
\subfloat[]{\includegraphics[width = 3in]{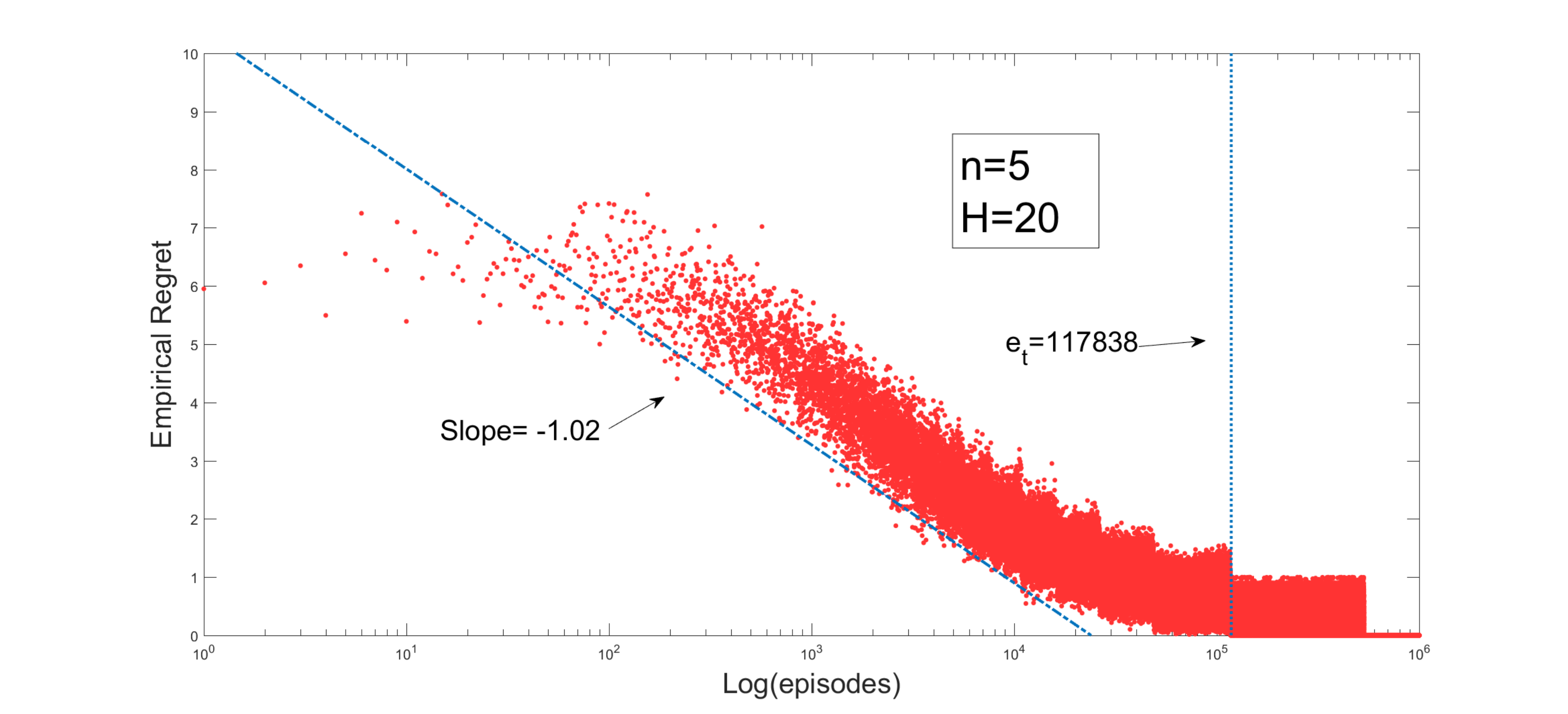}}\\
\subfloat[]{\includegraphics[width = 3in]{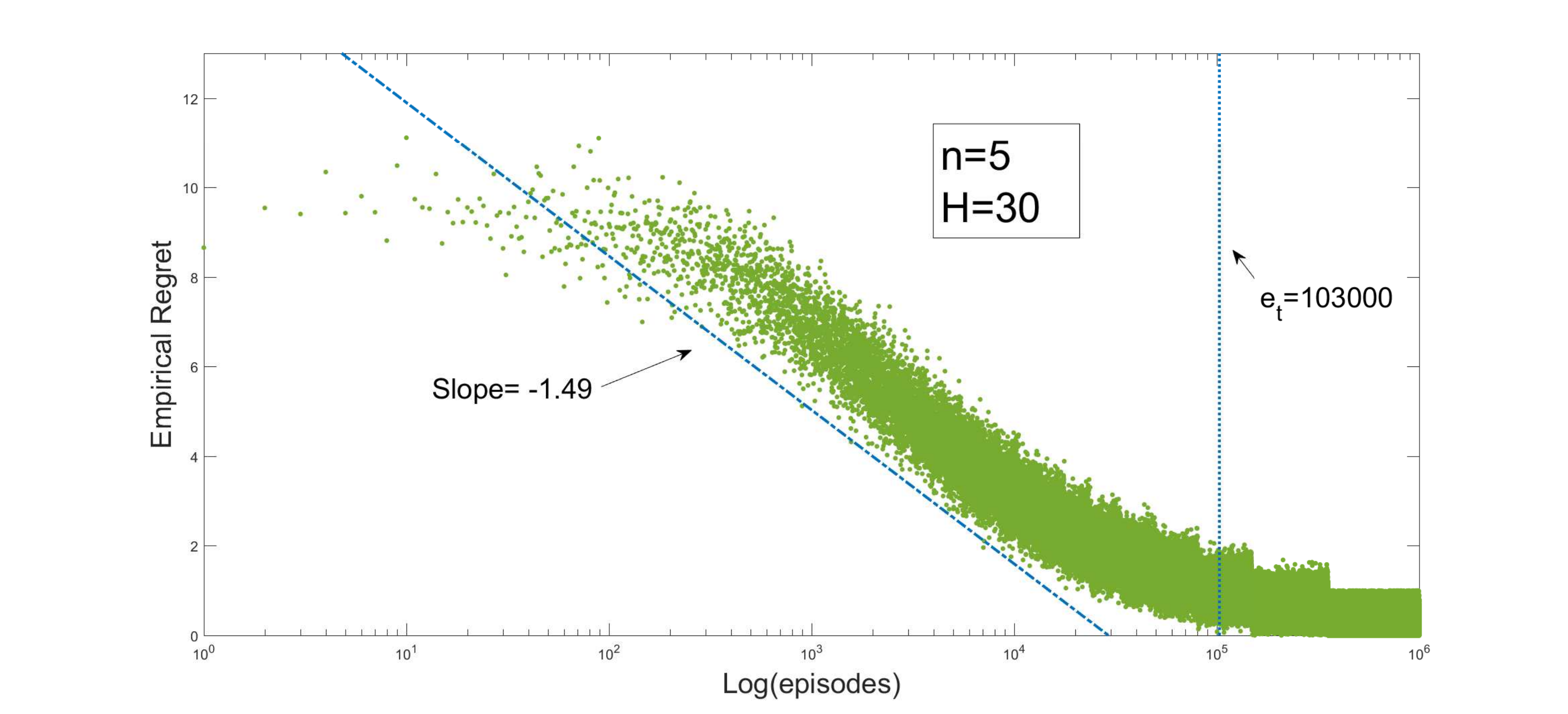}}\\
\subfloat[]{\includegraphics[width = 3in]{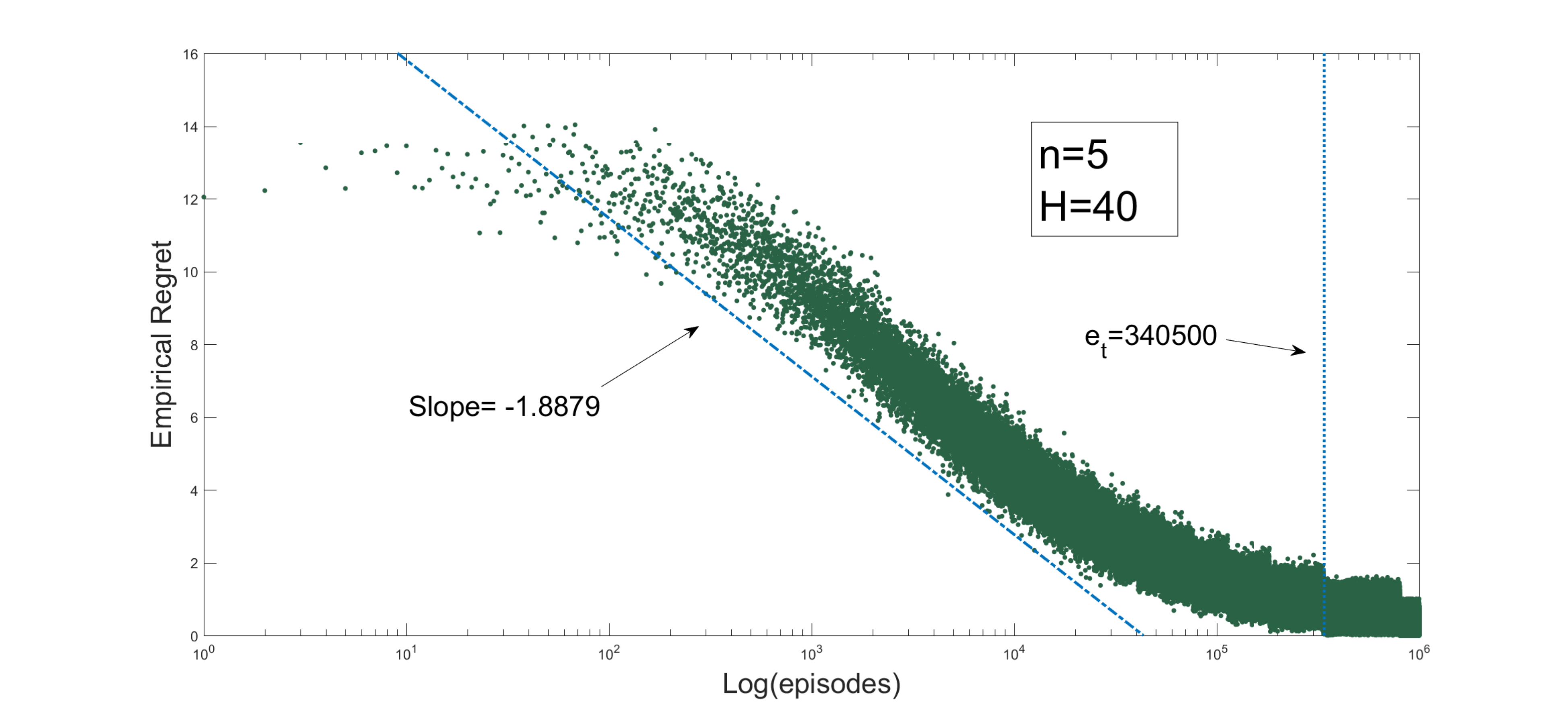}}
\caption{Plots (a), (b), (c) show the regret for different settings in the one-dimensional case.}
\end{figure}
\begin{figure}[H]
\ContinuedFloat
\subfloat[]{\includegraphics[width = 3in]{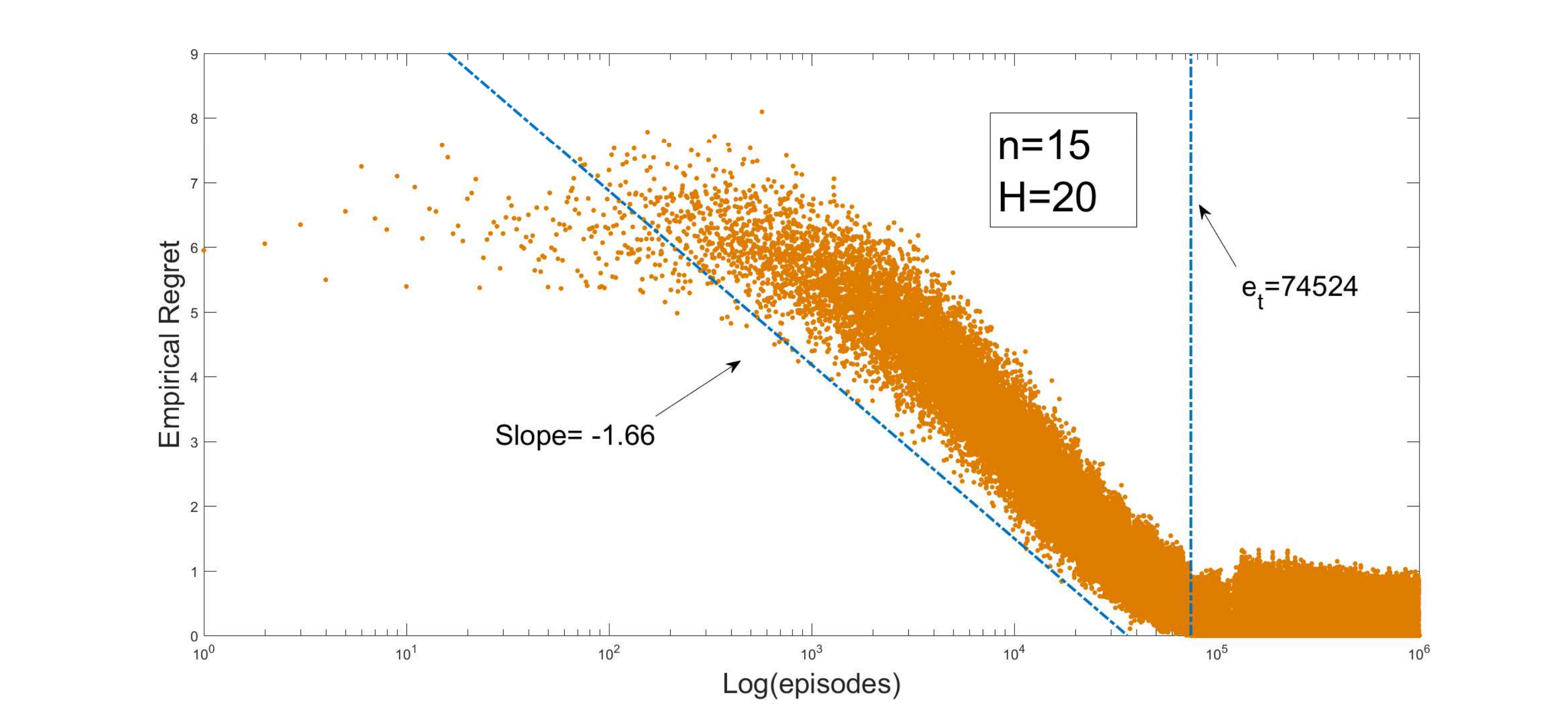}}\\
\subfloat[]{\includegraphics[width = 3in]{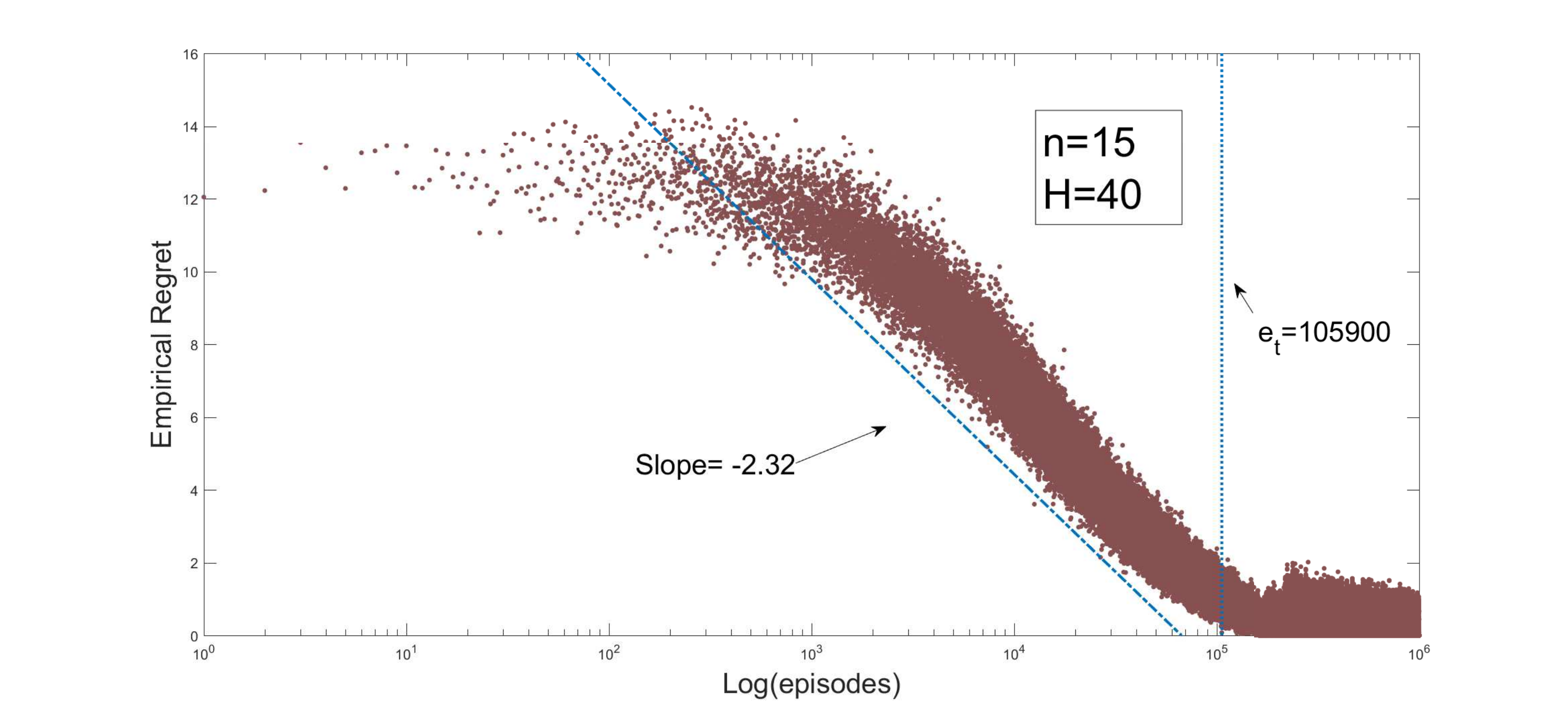}}\\
\subfloat[]{\includegraphics[width = 3 in]{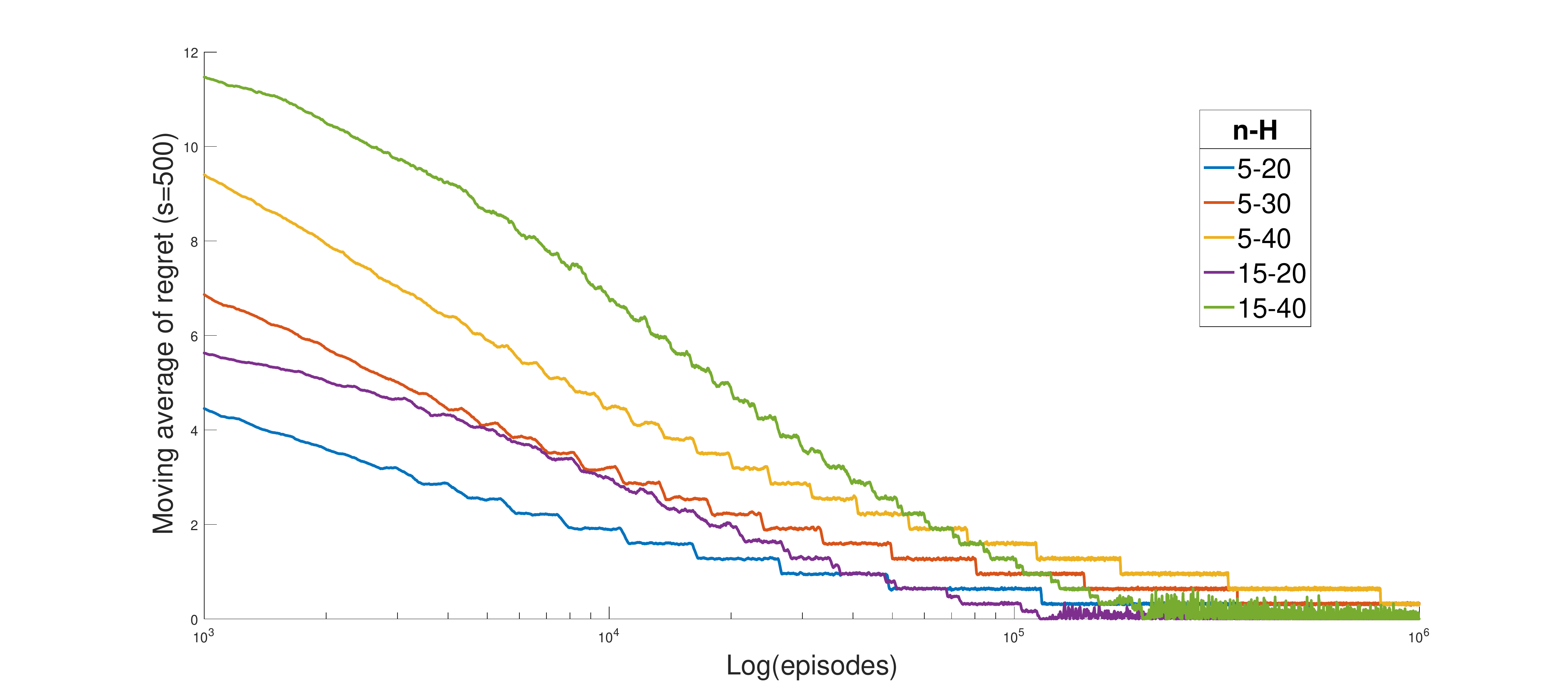}}
\caption{Plots  (d) and (e) show the regret for different settings in the one-dimensional case.
Plot (f) is obtained by taking the moving averages of regret for the same case. All the plots are log-scaled along the episode-axis.}
\label{fig:regret-1d}
\end{figure}



Next we conducted experiments on a two-dimensional problem with $A=2$ and a Lipschitz continuous reward function for the respective actions.
The reward functions were $\sin(\frac{x-y}{\sqrt{2}})+\cos(\frac{x+y}{\sqrt{2}})$ and 
$\sin(\frac{y-x}{\sqrt{2}})-\cos(\frac{x+y}{\sqrt{2}})$ for actions $a_1$ and $a_2$ respectively.
The analysis of this case is also similar to the one-dimensional setting. 
The algorithm shows similar trend for different values of $H$ for same number of intervals $n$ (see figure \ref{fig:regret-2d})
i.e, regret approaches zero slower for higher values of $H$ as in one-dimensional setting.


\section{Conclusion}
We considered the finite-horizon continuous reinforcement learning problem. We have given an algorithm based on UCBVI for the same. With the only
assumption that the reward function and transition probabilities are Lipscitz continuous we show that the upper bound on the discretization error is $const.Ln^{-\alpha}T$. We have shown the matching lower bound under the assumption that the algorithm discretizes the state-space of MDP.
In future we would like to show similar bound without this assumption. Also, with an additional assumption that 
the sampling is unbiased, we proved that the regret is of order $O(T^{2/3})$, when using our algorithm.

We have also given some experimental results to validate our propositions. In future we
would like to extend the algorithm to infinite-horizon case. This seems to be difficult as the UCBVI algorithm needs the horizon length $H$ as input.
We also want to improve the dependence on size of action $A$, Lipschitz constant $L$ 
and horizon length $H$ on the regret.
\bibliographystyle{ieeetr}
\bibliography{contl-lower}

\end{document}